\documentclass{article}

\usepackage[final]{neurips_2020}

\usepackage[utf8]{inputenc} 
\usepackage[T1]{fontenc}    
\usepackage[colorlinks=true,allcolors=blue]{hyperref}       
\usepackage{booktabs}       
\usepackage{amsfonts}       
\usepackage{nicefrac}       
\usepackage{microtype}      
\usepackage{soul}
\usepackage{xcolor}
\usepackage{amsthm}
\usepackage[ruled,vlined]{algorithm2e}
\usepackage{amsmath}
\usepackage{amssymb}
\usepackage{graphicx}
\usepackage{doi}

\newtheorem{corollary}{Corollary}
\newtheorem{proposition}{Proposition}

\newtheorem{lemma}{Lemma}
\SetKwInput{KwParams}{Parameters}              
\SetKwInput{KwOutput}{Output Cluster Averages} 
\SetKwInput{KwData}{Input Images}              

\DeclareMathOperator*{\argmin}{\arg\!\min}

\renewcommand{\v}{{\bf v}}
\renewcommand{\u}{{\bf u}}

\title{Wasserstein K-Means for Clustering\\Tomographic Projections}

\author{%
    Rohan Rao \\
    Princeton University\\
    \texttt{rohanr@princeton.edu } \\
    \And
    Amit Moscovich \\
    Princeton University\\
    \texttt{amit@moscovich.org} \\
    \And
    Amit Singer \\
    Princeton University \\
    \texttt{amits@math.princeton.edu} \\
}

\begin{document}

\maketitle

\begin{abstract}
Motivated by the 2D class averaging problem in single-particle cryo-electron microscopy (cryo-EM),
we present a k-means algorithm based on a rotationally-invariant Wasserstein metric for images.
Unlike existing methods  that are based on Euclidean ($L_2$) distances,
we prove that the Wasserstein metric better accommodates for the out-of-plane angular differences between different particle views.
We demonstrate on a synthetic dataset that our method gives superior results compared to an $L_2$ baseline.
Furthermore, there is little computational overhead, thanks to the use of a fast linear-time approximation to the Wasserstein-1 metric, also known as the Earthmover's distance.
\end{abstract}

\section{Introduction}
Single particle cryo-electron microscopy (cryo-EM) is a powerful method for reconstructing the 3D structure of individual proteins and other macromolecules \citep{Frank2006,Cheng2018}.
In a cryo-EM experiment, a sample of the molecule of interest is rapidly frozen, forming a thin sheet of vitreous ice, and then imaged using a transmission electron microscope.
This results in a set of images that contain noisy tomographic projections of the electrostatic potential of the molecule at random orientations.
The images then undergo a series of processing steps, including:
\begin{itemize}
    \item Motion correction.
    \item Estimation of the contrast transfer function (CTF) or point-spread function.
    \item Particle picking, where the position of the molecules in the ice sheet is identified.
    \item 2D class averaging, where similar particle images are clustered together and averaged.
    \item Filtering of bad images (typically, entire clusters from the previous step).
    \item Reconstruction of a 3D molecular model (or models).
\end{itemize}
All of these steps are done using special-purpose software, for example: \citet{TangEtal2007,Scheres2012b,LyumkisEtal2013,Xmipp2013,RohouGrigorieff2015,PunjaniEtal2017,HeimowitzAndenSinger2018,GrantRohouGrigorieff2018}.

Cryo-EM has been gaining popularity and its importance has been recognized in the 2017 Nobel prize in Chemistry \citep{CryoEMNobel2017}. An examination of the Protein Data Bank \citep{BermanEtal2000} shows that in 2020, 16\% of entries had structures that were determined by cryo-EM, compared to 12\% in 2019, and 7\% in 2018. 
For more on cryo-EM and its current challenges, see the reviews of \cite{VinothkumarHenderson2016,Lyumkis2019,SingerSigworth2020}.

\subsection{Image formation in cryo-EM}

Given a 3D molecule with an electrostatic potential function $\rho: \mathbb{R}^3 \rightarrow \mathbb{R}$ (henceforth ``particle''), the images captured by the electron microscope are modeled as noisy tomographic projections along different orientations.
A tomographic projection of the particle onto the x-y plane is defined as $\left(\mathcal{T}\rho\right)(x,y) = \int_\mathbb{R} \rho(x,y,z)dz$.
More general tomographic projections are given by a rotation $R$ of the potential function $\rho$ followed by projection onto the x-y plane, 
\begin{equation}
    \mathcal{T}_R \rho = \int_\mathbb{R} \left( R \rho \right)(x,y,z)dz
\end{equation}
where $R \in \text{SO}(3)$, and we define $(R\rho)(x,y,z) = \rho(R^T(x,y,z))$.
Particle images in cryo-EM are typically modeled as follows,
\begin{align}
    \text{Image} = h * \mathcal{T}_{R}\rho + {\boldsymbol \epsilon},
\end{align}
where $h$ is a point spread function that is convolved with the projection and ${\boldsymbol \epsilon}$ is Gaussian noise.
See Figure \ref{fig:experiments-3D-ribosome} for an example of a tomographic projection. We define the viewing angle of our particle as a unit vector $\v$ that represents orientation of our particle $R \in \mathrm{SO}(3)$ modulo in-plane rotations (which can be represented as members of $\mathrm{SO}(2)$).

\subsection{Clustering tomographic projections}
The imaging process in cryo-EM involves high-levels of noise.
To improve the signal-to-noise ratio (SNR), it is common to cluster the noisy projection images with ones that are similar up to an in-plane rotations.
This clustering task is called ``2D classification and averaging'' in the cryo-EM literature and the clusters are called classes.
The results of this stage have multiple uses, including:
\begin{itemize}
    \item Template selection from 2D class averages for particle picking \citep{FrankWagenkrecht1983,ChenGrigorief2007,Scheres2014}.
    \item Particle sorting \citep{ZhouMoscovichBendoryBartesaghi2020}, discarding images that belong to bad clusters.
    \item Visual assessment and symmetry detection.
    \item Ab-initio modeling, where an initial 3D model (or set of models) is constructed to be refined in later stages \citep{GreenbergShkolnisky2017,PunjaniEtal2017}. 
\end{itemize}
Existing methods for 2D clustering include the reference-free alignment algorithm that tries to find a global alignment of all images \citep{PenczekRadermacherFrank1992}, clustering based on invariant functions and multivariate statistical analysis \citep{SchatzVanheel1990}, rotationally invariant k-means \citep{PenczekZhuFrank1996} and the Expectation-Maximization (EM) based approach of \cite{ScheresValleCarazo2005}.
All of these methods use the $L_2$ distance metric on rotated images.

Let us consider a simple centroid+noise model, where each image in a cluster is generated by adding Gaussian noise to a particular clean view ${\boldsymbol \mu}$ that is the (oracle) centroid,
\begin{align}
    \text{Image}_i \sim \mathcal{N}({\boldsymbol \mu}, \sigma^2 I).
\end{align}
In that case the $L_2$ metric is (up to constants) nothing but the log-likelihood of the image patch, conditioned on $\boldsymbol \mu$,
\begin{align}
    \log \mathcal L(\text{Image}_i | {\boldsymbol \mu}) = -\|\text{Image}_i - {\boldsymbol \mu}\|^2/2 \sigma^2 + C_\sigma.
\end{align}
Under the centroid+noise model, the commonly-used $L_2$ distance metric seems perfectly suitable.
However, real particle images have many other sources of variability, including  angular differences, in-plane shifts and various forms of molecular heterogeneity.
For these sources of variability, the $L_2$ distance metric is ill-suited.
See Section \ref{sec:theory} for more on this.

In this work we propose to use the Wasserstein-1 metric $W_1$, also known as the Earthmover's distance, as an alternative to the $L_2$ distance for comparing particle images.
In Section \ref{sec:methods} we describe a variant of the rotationally invariant k-means that is based on a fast linear-time approximation of $W_1$ and in Section \ref{sec:results} we demonstrate superior performance on a synthetic dataset of Ribosome projections.
In Section \ref{sec:theory} we analyze the behavior of the $W_1$ and $L_2$ metrics theoretically with respect to angular differences of tomographic projections.
In particular we show that the rotationally-invariant $W_1$ metric has the nice property that it is bounded from above by the angular difference of the projections.
On the other hand, the $L_2$ metric shows no such relation.

\section{Methods} \label{sec:methods}

In cryo-EM, in-plane rotations of the molecular projections are assumed to be uniformly distributed, hence is is desirable for the distance metric to be invariant to in-plane rotations.
A natural candidate is the rotationally-invariant Euclidean distance,
\begin{equation}
    L_2^R(I_1, I_2) := \min_{R \in \mathrm{SO}(2)} ||I_1 - R I_2||_2.
    \label{eq:rotationally_invariant_L2}
\end{equation}
A drawback of this metric is that visually similar projection images that have a small viewing angle difference can have an arbitrarily large $L_2$ distance.
See discussion in Section \ref{sec:theory}.

To address this, we define a metric based on the Wasserstein-$p$ Metric  between two probability distributions \citep{Villani2003}. This metric measures the ``work'' it takes to transport the mass of one probability distribution to the other, where work is defined as the amount of mass times the distance (on the underlying metric space) between the origin and destination of the mass. More formally, the Wasserstein-$p$ distance between two normalized greyscale $N \times N$ images is defined as

\begin{equation}
    W_p(I_1, I_2) = \inf_{\pi \in \Pi(I_1, I_2)} \sum_{u \in [N]^2}\sum_{v \in [N]^2}||u-v||^p\pi(u, v),
\end{equation}
where $\Pi(I_1, I_2)$ is the set of joint distributions on $[N]^2$ with marginals $I_1, I_2$ respectively.

In cryo-EM, the Wasserstein-1 metric has been used to understand the conformation space of 3D volumes of molecules \citep{ZeleskoMoscovichKileelSinger2020}. For clustering tomographic projections, we construct a rotationally-invariant $W_p$ distance to be our clustering metric analogously to Eq. \eqref{eq:rotationally_invariant_L2}
\begin{equation} \label{def:rotinv_wasserstein}
    W_p^R(I_1, I_2) := \min_{R \in \mathrm{SO}(2)} W_p(I_1, R I_2).
\end{equation}

\subsection{Clustering with rotationally-invariant metrics}

Algorithm \ref{algo:k-means} describes a generic rotationally-invariant k-means algorithm based on an arbitrary image patch distance metric $d$.
The choice $d = L_2$ gives the rotationally-invariant k-means of \cite{PenczekZhuFrank1996}.
By supplying $d = W_1$ we get a new rotationally-invariant k-means algorithm based on the $W_1$ distance. We choose $p=1$ for the $W_p$ distance as it admits a fast linear-time wavelet approximation as we will see in the next section.
For both choices of the metric, we initialize the centroids using a rotationally-invariant k-means++ initialization which we describe in Algorithm \ref{algo:k++}. When we denote $r \in \mathrm{rot}$ we mean that $r$ performs an in-plane rotation of an image. We approximate the space of all in-plane rotation angles by a discrete set of angles.

\begin{algorithm}[ht]
\SetAlgoLined
\DontPrintSemicolon
\KwParams{$k, n, d$}
\;\vspace{-10pt}
\KwOutput{$\{C_j\}_{j \in [k]}$}
\;\vspace{-11pt}
\KwData{$\{I_i\}_{i \in [n]}$}
\;\vspace{-5pt}
$\{C_j\}_{j \in [k]} := \mathrm{InitializeCenters}(\{I_i\}_{i \in [n]})$\\
\;\vspace{-9pt}

\While{\textnormal{loss decreases}}{
    $\mathrm{loss} = 0$\\
    \For{$i \in [n]$} {
        $(j_i, r_i) := \argmin_{(j,r):j \in [k], r \in \mathrm{rot}} d(I_i, r(C_j))$\\
        $\mathrm{loss} = \mathrm{loss} + d(I_i, r_i(C_{j_i}))^2$\\
    }
    \For{$j \in [k]$}{
        $C_j := \mathrm{mean}(\{r_i^{-1}(I_i) : j_i = j\})$\\
    }
}
\;\vspace{-5pt}
\KwRet$\{C_j\}_{j \in [k]}$

 \caption{Rotationally invariant k-means}
 \label{algo:k-means}
 \end{algorithm}

 \begin{algorithm}[ht]
 \SetAlgoLined
\SetAlgoLined
\DontPrintSemicolon
\KwParams{$k, n, d$}
\;\vspace{-10pt}
\KwOutput{$\{C_j\}_{j \in [k]}$}
\;\vspace{-11pt}
\KwData{$\{I_i\}_{i \in [n]}$}
\;\vspace{-5pt}
$C_1 = \mathrm{RandomSelect}(\{I_i\}_{i \in [n]}$)\\[1pt]
\For{$j \in \{2, ..., k\}$}{
    \For{$i \in [n]$}{
        $p_{i} := \left( \min_{m \in [j-1], r \in \mathrm{rot}} d(I_i, r(C_m)) \right)^2$
    }
    $j = \mathrm{DrawFrom}({\bf p})$
    \tcp{Draw index $j$ with probability proportional to $p_j$}
    $C_j = I_j$
}
\;\vspace{-5pt}
\KwRet $\{C_j\}_{j \in [k]}$
\caption{Rotationally invariant k-means++ initialization}
 \label{algo:k++}
 \end{algorithm}

\subsection{Computing the Earthmover's distance}
Computing the distance $W_1(I_1, I_2)$ (also known as the Earthmover's distance) between two $N \times N$ pixels can be formulated as a linear program in $O(N^2)$ variables and $O(N)$ constraints. Given $n$ images, $k$ centers, and $m$ rotations for each image and $t$ iterations of k-means, we have to compute $O(n\times m \times k \times t)$ distances. Computing the $W_1$ distance between two $100 \times 100$ size images using a standard LP solver takes on average of 5 seconds to compute using the Python Optimal Transport Library of \cite{FlameryCourty2017} on a single core of a 7th generation 2.8 GHz Intel Core i7 processor. Computing the exact $W_1$ distance over all the rotations of all the images over all the iterations is prohibitively slow.

Fortunately, the $W_1$ distance admits a wavelet-based approximation with a runtime that is linear in the number of pixels \citep{ShirdhonkarJacobs2008}. Let $\mathcal{W}I_i$ be the 2D wavelet transform of $I_i$. We can approximate the $W_1$ distance between two images $I_1, I_2$ by a weighted $\ell_1$ distance between their wavelet coefficients,
\begin{align} \label{eq:approximate_emd}
    \widetilde{W_1}(I_1, I_2) := \sum_{\lambda} 2^{-2 \lambda_s}|(\mathcal{W}I_1)(\lambda) - (\mathcal{W}I_2)(\lambda)|,
\end{align}
where  $\lambda$ goes over all the wavelet coefficients and $\lambda_s$ is the scale of the coefficient.
This metric is strongly equivalent to $W_1$. i.e. there exist constants $0 < c < C$ such that for all $I_1, I_2 \in \mathbb{R}^{N^2}$ 
\begin{equation}
    c \cdot \widetilde{W_1}(I_1, I_2) \leq W_1(I_1, I_2) \leq C \cdot \widetilde{W_1}(I_1, I_2).
\end{equation}
Different choices of the wavelet basis give different ratios $C/c$. We chose the symmetric Daubechies wavelets of order 5 due to the quality of their approximation \citep{ShirdhonkarJacobs2008}.
The sum in Eq. \eqref{eq:approximate_emd} was computed with scale up to $6$ using the PyWavelets package \citep{LeeEtal2019}.
\section{Experimental results} \label{sec:results}
\subsection{Dataset generation}
We built a synthetic dataset of 10,000 tomographic projections of the Plasmodium falciparum 80S ribosome bound to the anti-protozoan drug emetine, EMD-2660  \citep{WongEtal2014}, shown in Figure~\ref{fig:experiments-3D-ribosome}.
To generate each image, we randomly rotated the ribosome around its center of mass using a uniform draw of $\mathrm{SO}(3)$, projected it to 2D by summing over the $z$ axis and resized the resulting image to $128 \times 128$.
Finally we added i.i.d. $\mathcal{N}(0,\sigma^2)$ noise to the image at different signal-to-noise (SNR) levels. Given a dataset of images $S \in \mathbb{R}^{N \times N \times n} $, the SNR is defined as
\begin{equation}
    \text{SNR} := \frac{ \|D\|^2}{nN^2\sigma^2}.
\end{equation}
We produce three datasets to run experiments on by adding noise at SNR values $\{1/8, 1/12, 1/16\}$.

\begin{figure}[h]
    \centering
    \includegraphics[width=4cm]{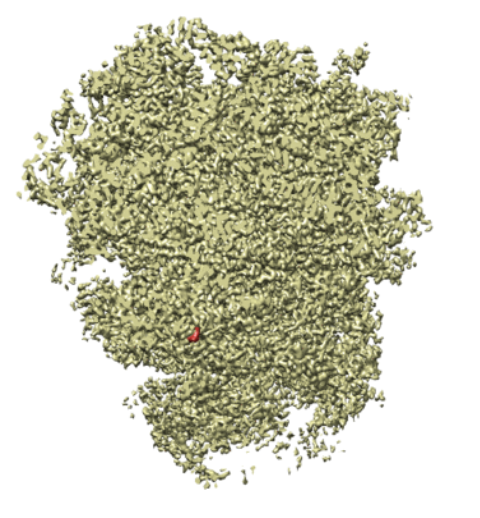}
    \includegraphics[width=9cm]{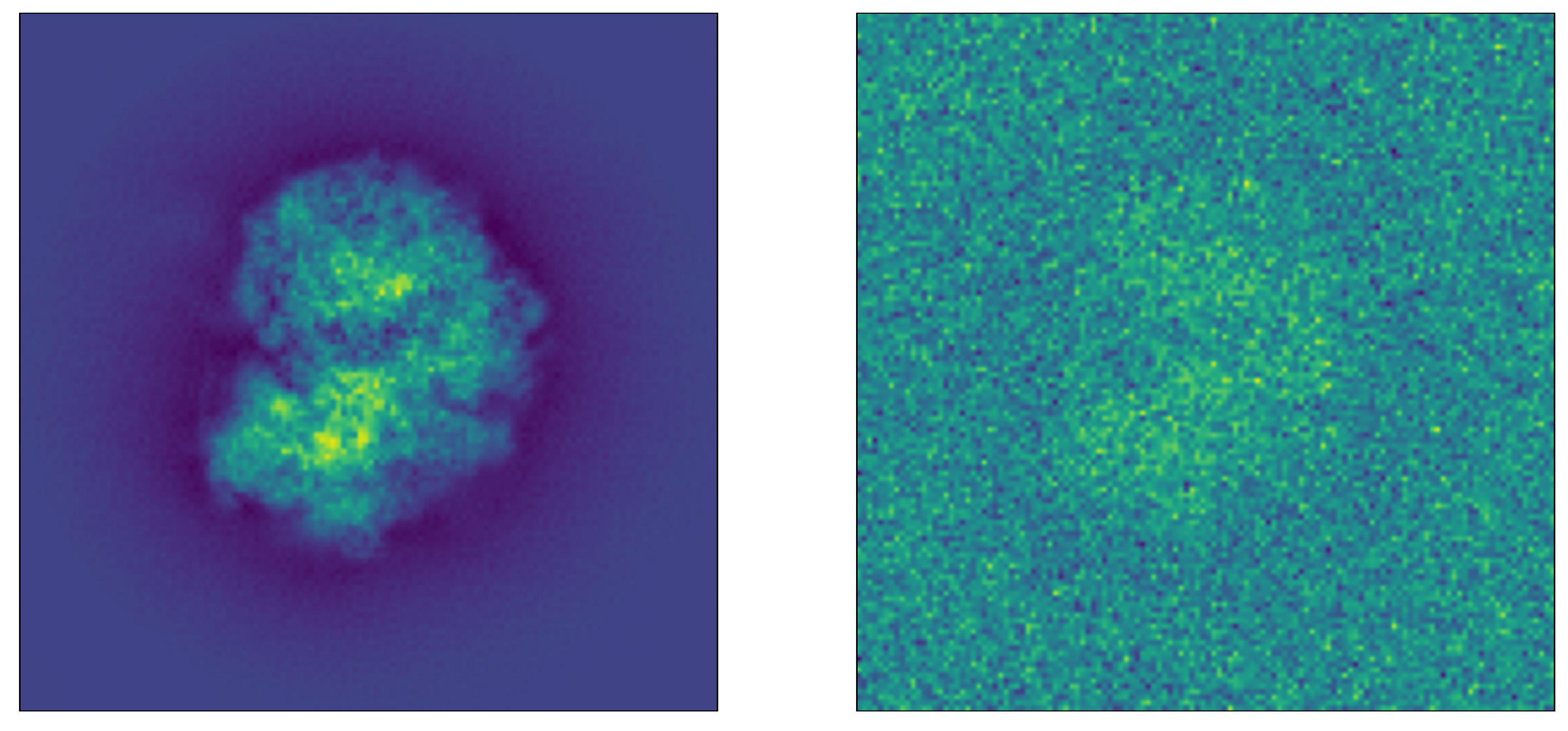}
    \caption{(left) Surface plot of the 3D electrostatic potential of the 80S ribosome; (middle) Example tomographic projection; (right) The same projection with Gaussian noise at SNR=1/16.}
    \label{fig:experiments-3D-ribosome}
\end{figure}

\subsection{Simulation results} \label{sec:results}
We performed rotationally-invariant 2D clustering on our three datasets using rotationally-invariant k-means++ (Algorithms \ref{algo:k-means}, \ref{algo:k++}) with the $W_1$ and $L_2$ distance using $k=150$ clusters.

All in-plane rotations at increments of $\pi/100$ radians were tested in the Algorithm \ref{algo:k-means}. 
To quantify the difference in performance, we computed the distribution of within-cluster viewing angle differences for all pairs of images assigned to the same cluster. For a molecule like the Ribosome that has no symmetries, we expect these angular differences to be concentrated around zero, since large angular differences typically give large differences in the projection images.
For all SNR levels, we see that $W_1$ gives better angular coherence than $L_2$.
Note that some of these distributions have a small peak towards 180 degrees which physically represents our algorithms confusing
antipodal viewing angles.

\begin{figure}
    \includegraphics[width=0.33\textwidth]{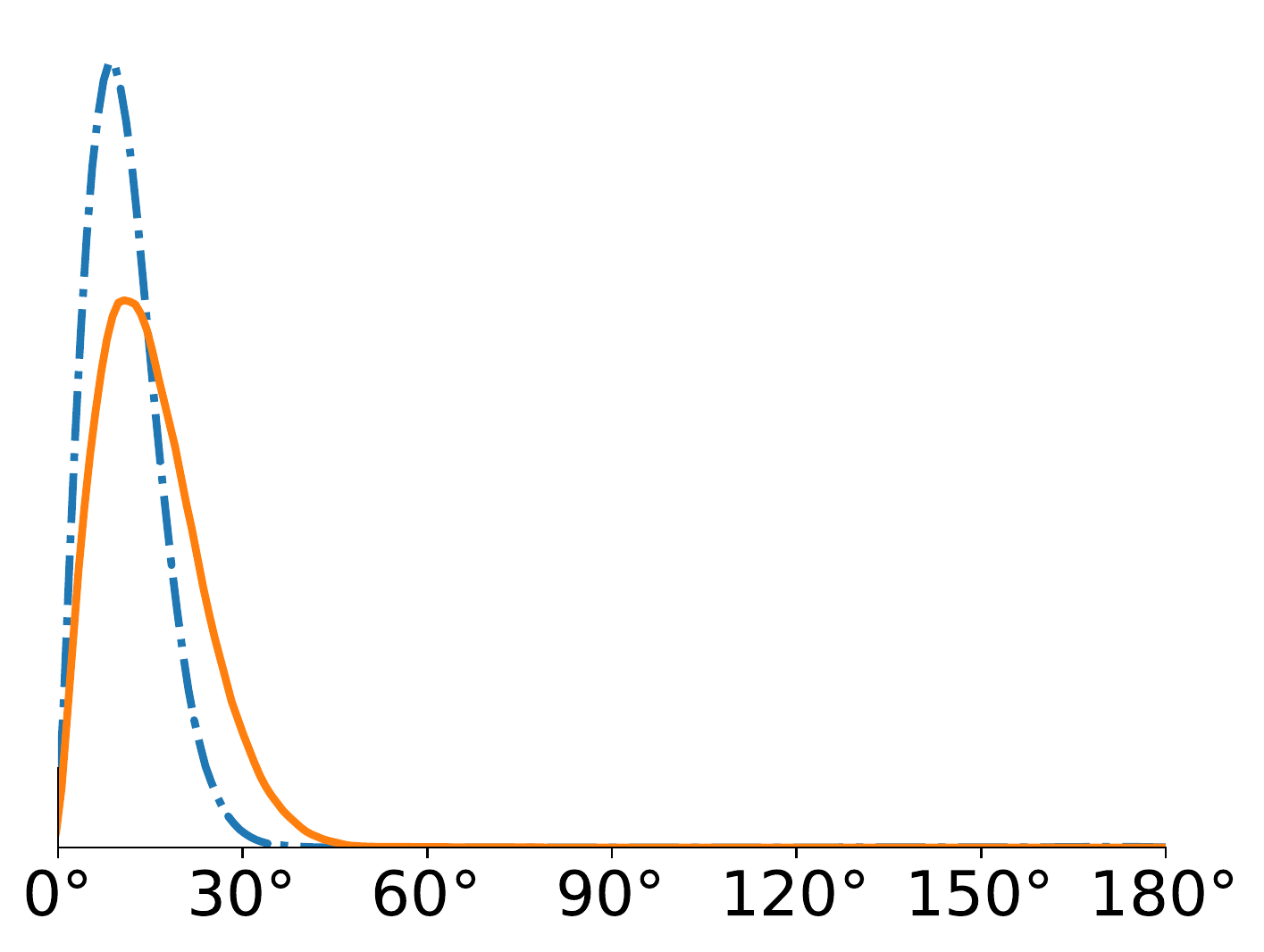}
    \includegraphics[width=0.33\textwidth]{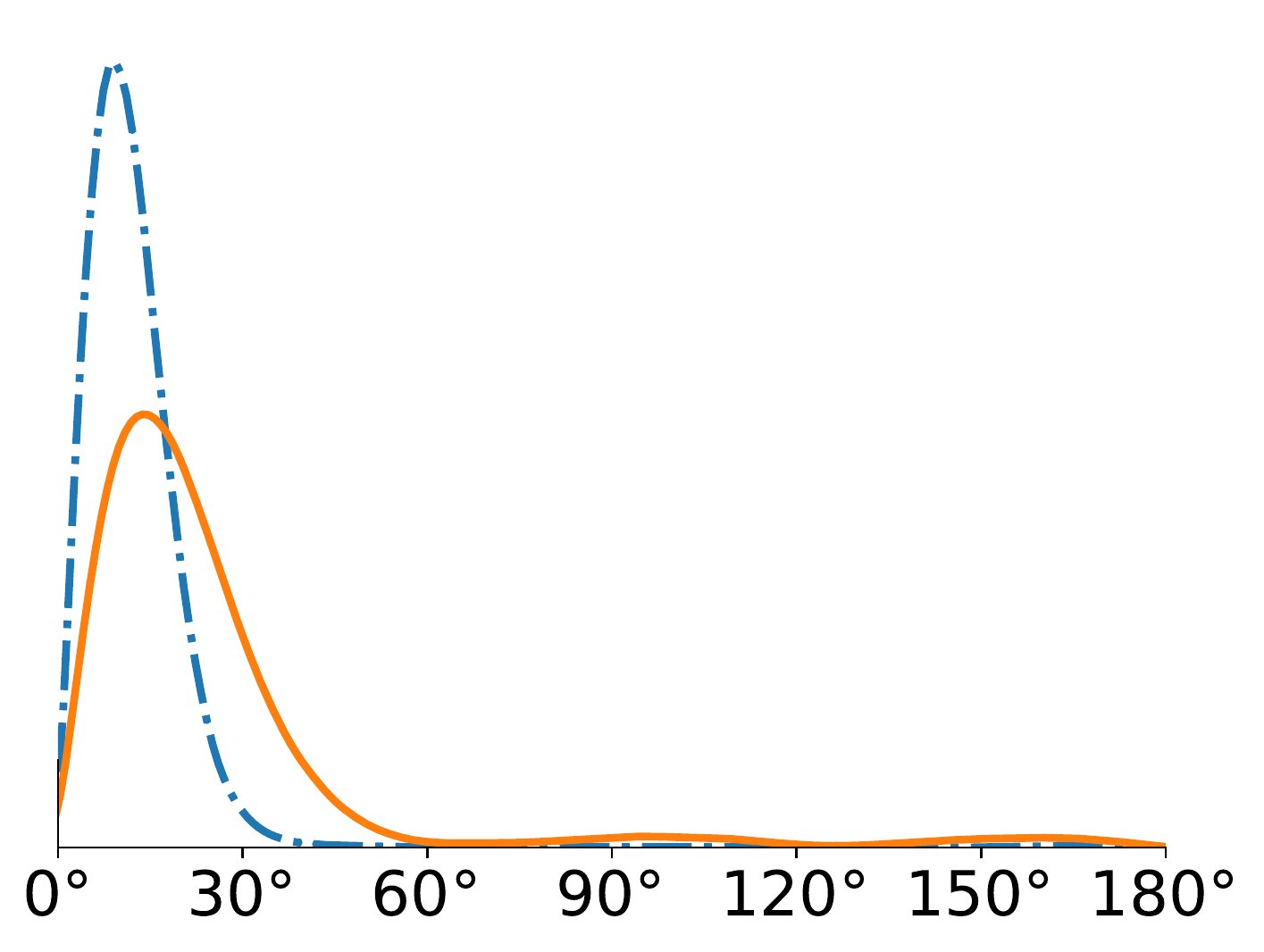}
    \includegraphics[width=0.33\textwidth]{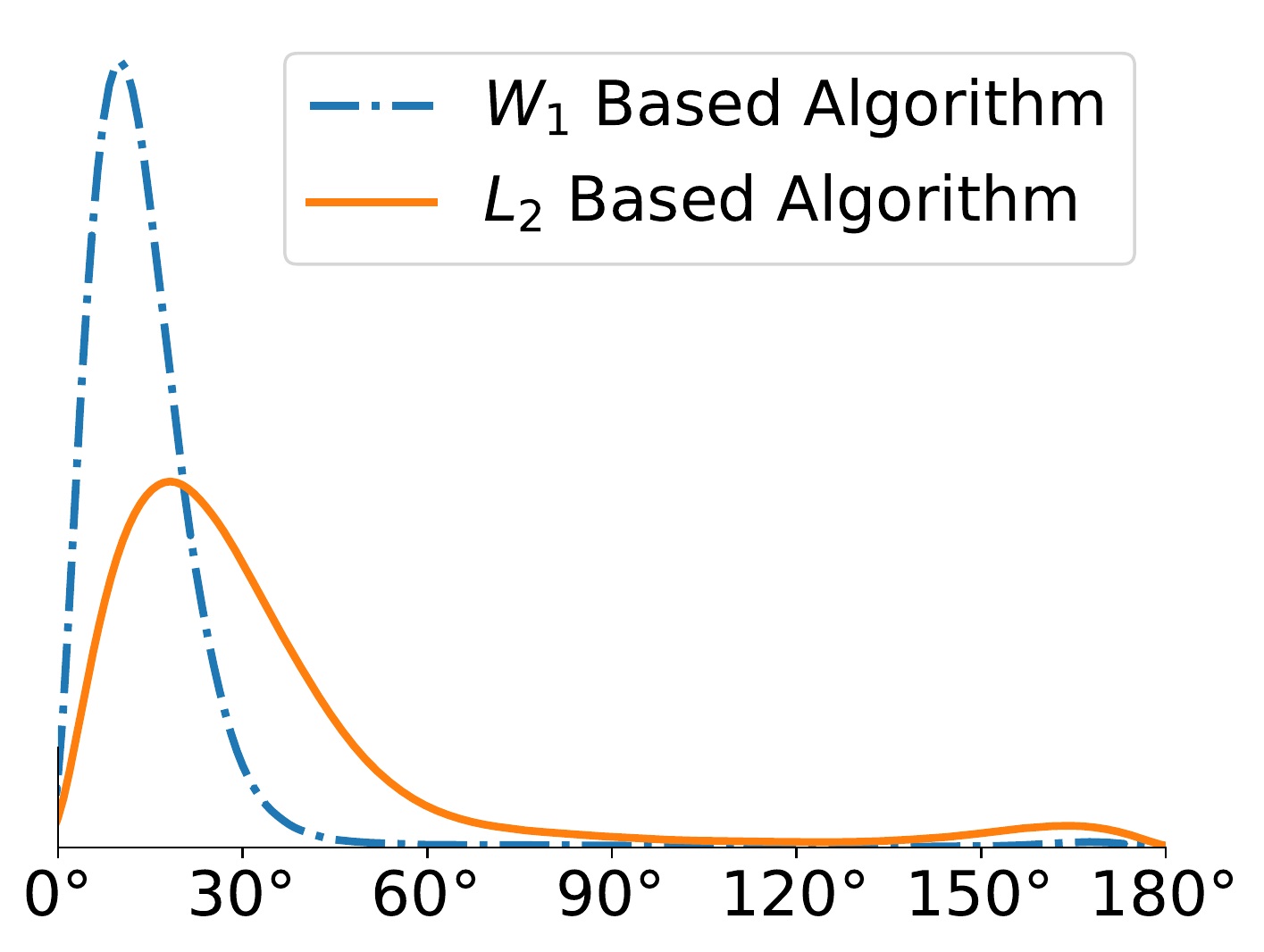}
\caption{Distribution of within-cluster pairwise angular differences (narrower is better) at SNR values $1/8,\ 1/12,\ 1/16$ from left to right.}
\label{fig:angulardistribution}
\end{figure}

In Figure \ref{fig:ribosome_centroids} we show the cluster means  of the 8 largest clusters at SNR $1/16$. Visually, we can see that the algorithm based on  $W_1$ produces sharper mean images.
\begin{figure}
    \centering
    \includegraphics[width=\textwidth]{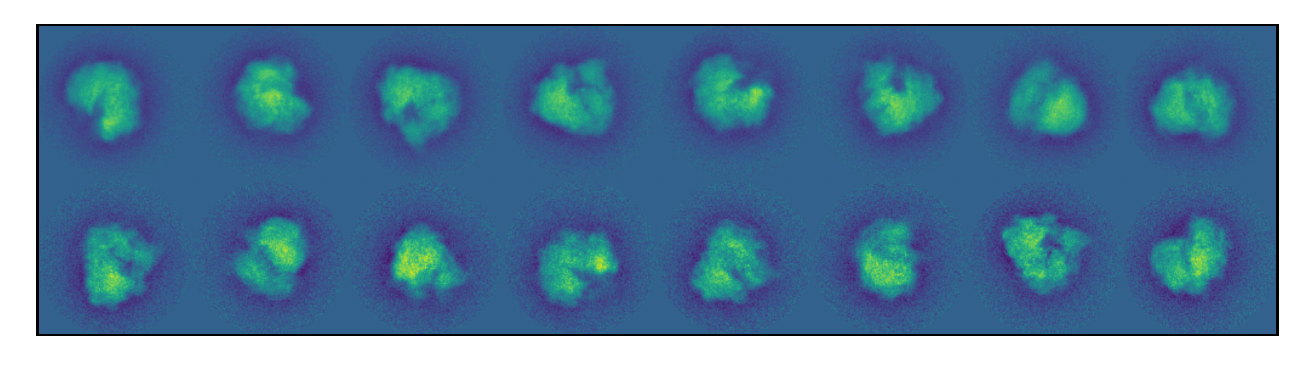}
    \caption{Means of the largest 8 clusters for the SNR=$1/16$ dataset, sorted by cluster size from left to right.
    (top)  $L_2$ distance; (bottom) $W_1$ metric. We can see that the $W_1$ metric preserves more details than the $L_2$ distance. The portion of the image that does not contain our particle appears less noisy in the $L_2$ averages, however this is just an artifact of the cluster size. By averaging a large number of images, the noise decreases but the signal also deteriorates.}
    \label{fig:ribosome_centroids}

\end{figure}
Finally, we examine the occupancy of the clusters. The $W_1$ algorithm provides more evenly sized clusters, whereas for the $L_2$ algorithm we see a few very populated centers and a large dropoff in occupancy in the other clusters.

When clustering images with Gaussian noise, the averages of larger clusters will tend to be less noisy, since the noise variance is inversely proportional to the cluster size. Due to the lower noise levels, more of the images will be assigned to the larger clusters, making them even larger.
This ``rich get richer'' phenomenon has been observed in cryo-EM \citep{SorzanoEtal2010}.
It can explain the large differences in occupancy visible in the top panels of Figure \ref{fig:occupancy}, despite the fact that the angles were drawn uniformly.
The Wasserstein distance is more resilient to i.i.d. noise and this may explain the uniformity in the resulting cluster sizes seen in the bottom panels of Figure \ref{fig:occupancy}.

Finally, clustering with the $W_1$ distance does not increase the runtime of the clustering algorithm by much. We include timing results per iteration of k-means with each metric, along with the number of iterations it took for the algorithm to converge in Table \ref{table:k-means-timing}.
\begin{figure}
    \includegraphics[width=0.33\textwidth]{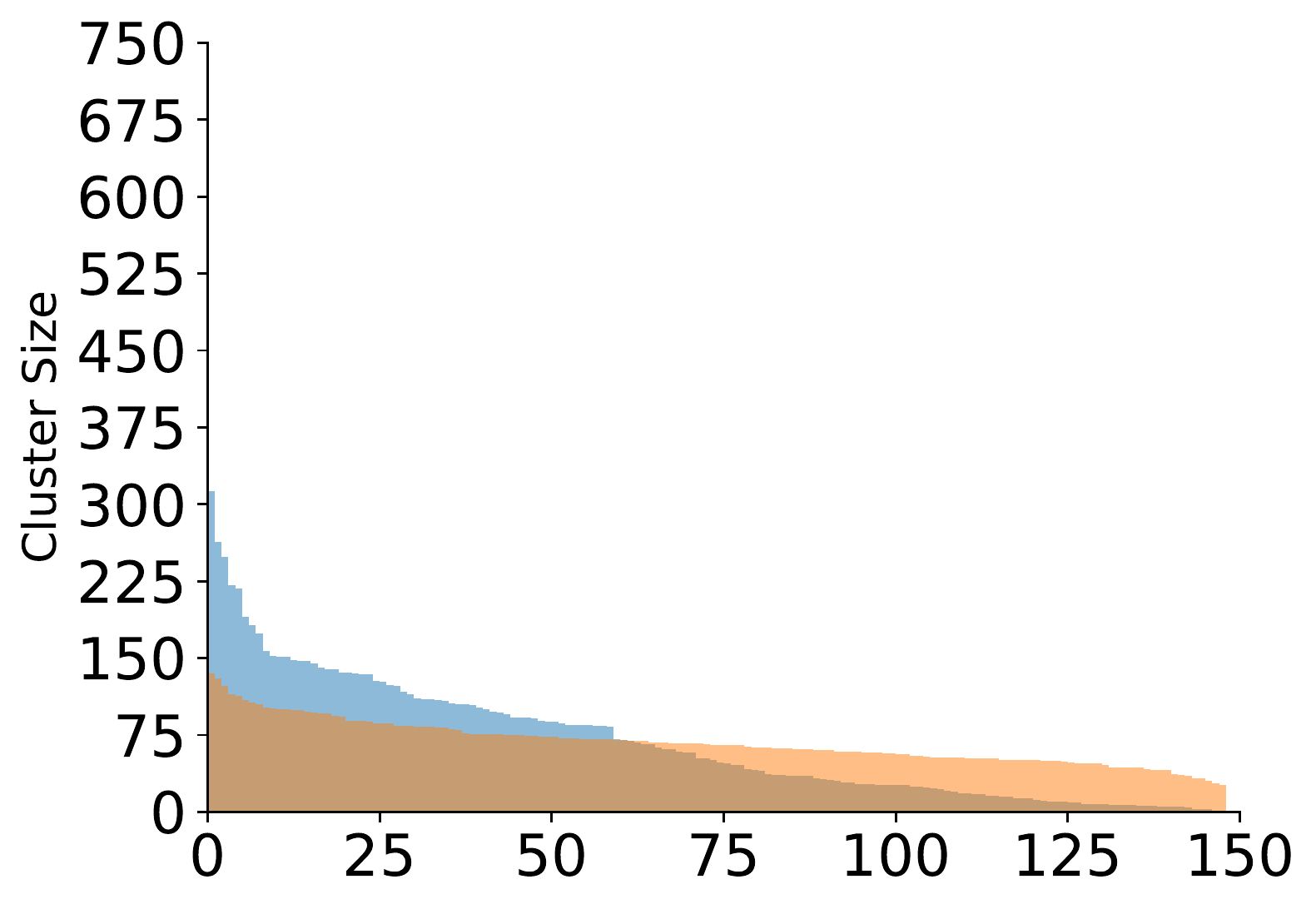}
    \includegraphics[width=0.33\textwidth]{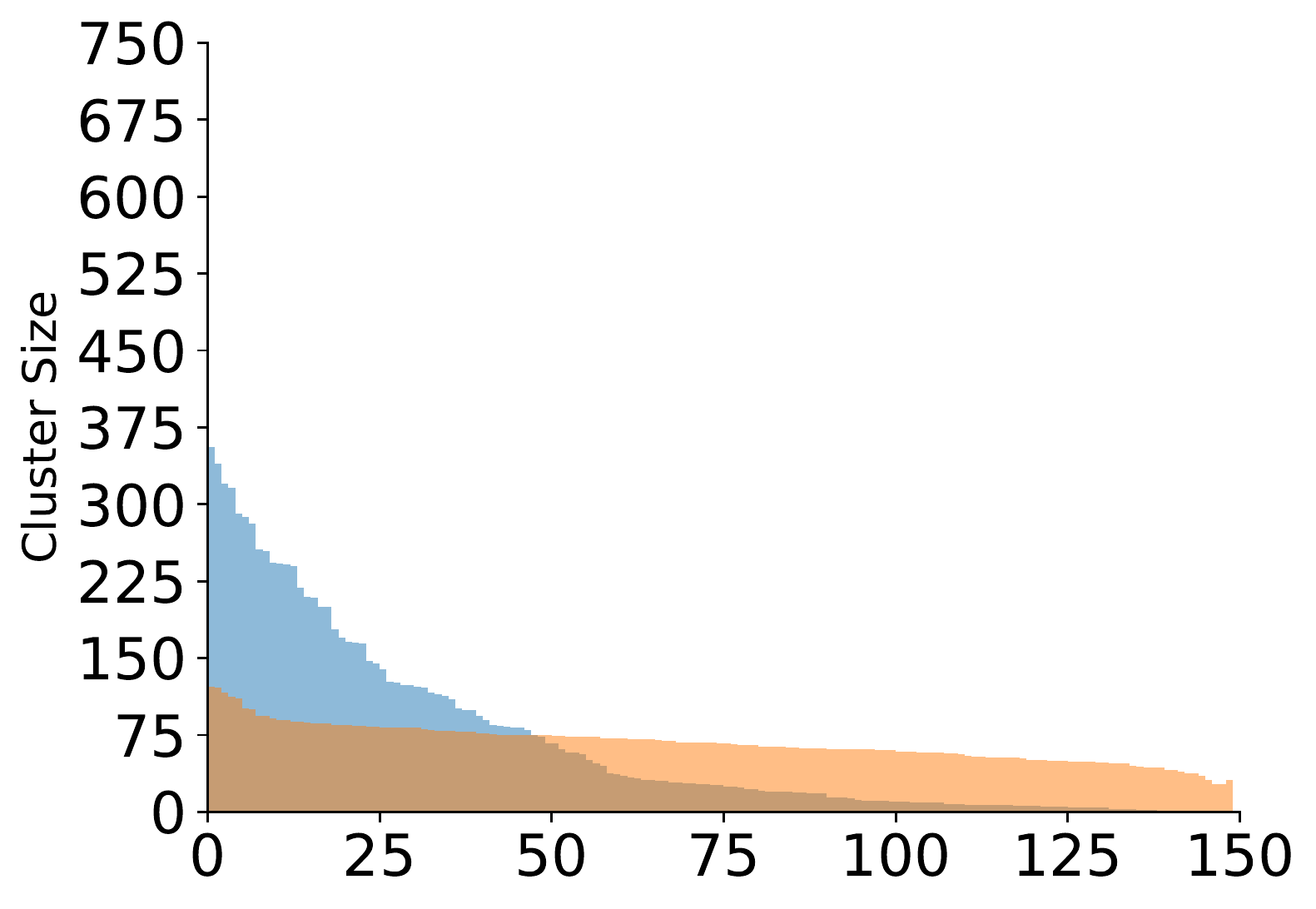}
    \includegraphics[width=0.33\textwidth]{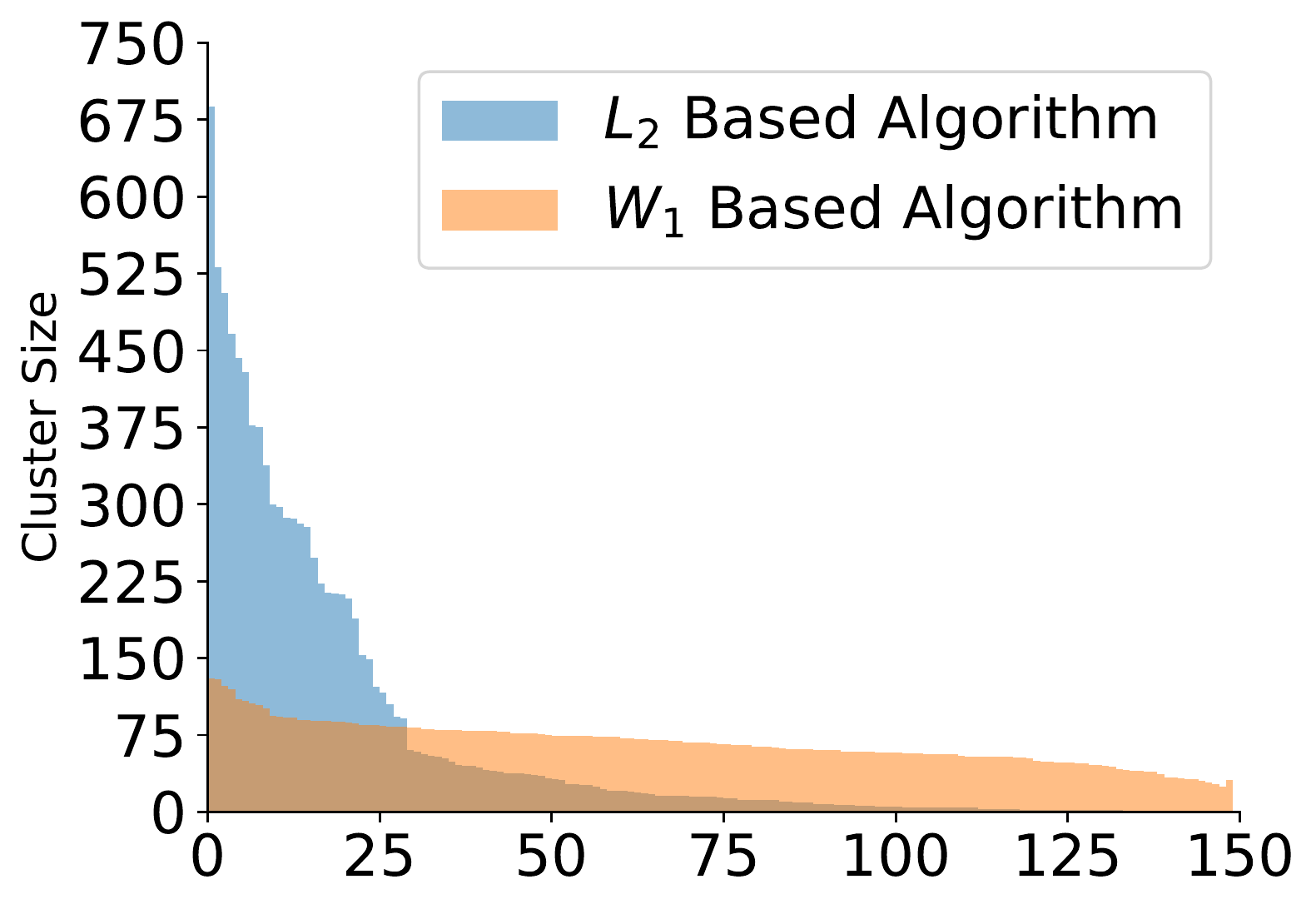}
    \caption{Cluster sizes for the datasets with signal-to-noise $1/8,\, 1/12,\, 1/16$ from left to right.} 
    \label{fig:occupancy}
\end{figure}

\begin{table}[h]
\caption{Seconds per iteration averaged (over two runs) for the $L_2$ and $W_1$ metrics and the number of iterations before
convergence at different SNRs (for one run). The programs were run using 32 Intel core i7 CPU cores which allow us to compute the distances from images in our dataset to all the cluster averages in a parallelized fashion across the averages.}
\label{table:k-means-timing}
\centering
\begin{tabular}{rcccc}
Metric & k-means (sec/iteration) & $n_{\text{iter}}$ (SNR $=1/8$) & $n_{\text{iter}}$ (SNR $=1/12$) & $n_{\text{iter}}$ (SNR $=1/16$)\\
 \toprule
 $L_2$ & 1204 s & 13 & 31 & 22\\
 $W_1$ & 1676 s & 23 & 27 & 26\\
\end{tabular}
\end{table}

\subsection{Sensitivity to noise}
To examine the effect of noise on the $W_1^R$ and $L_2^R$ distances, we plot the $W_1^R$ and $L_2^R$ distances against the viewing angle difference between projections of the ribosome. In Figure \ref{fig:noise} we can see that $W_1^R$ continues to give meaningful distances under noise compared to the $L_2^R$ distance.
\begin{figure}[h]
\centering
    \includegraphics[width=0.4\textwidth]{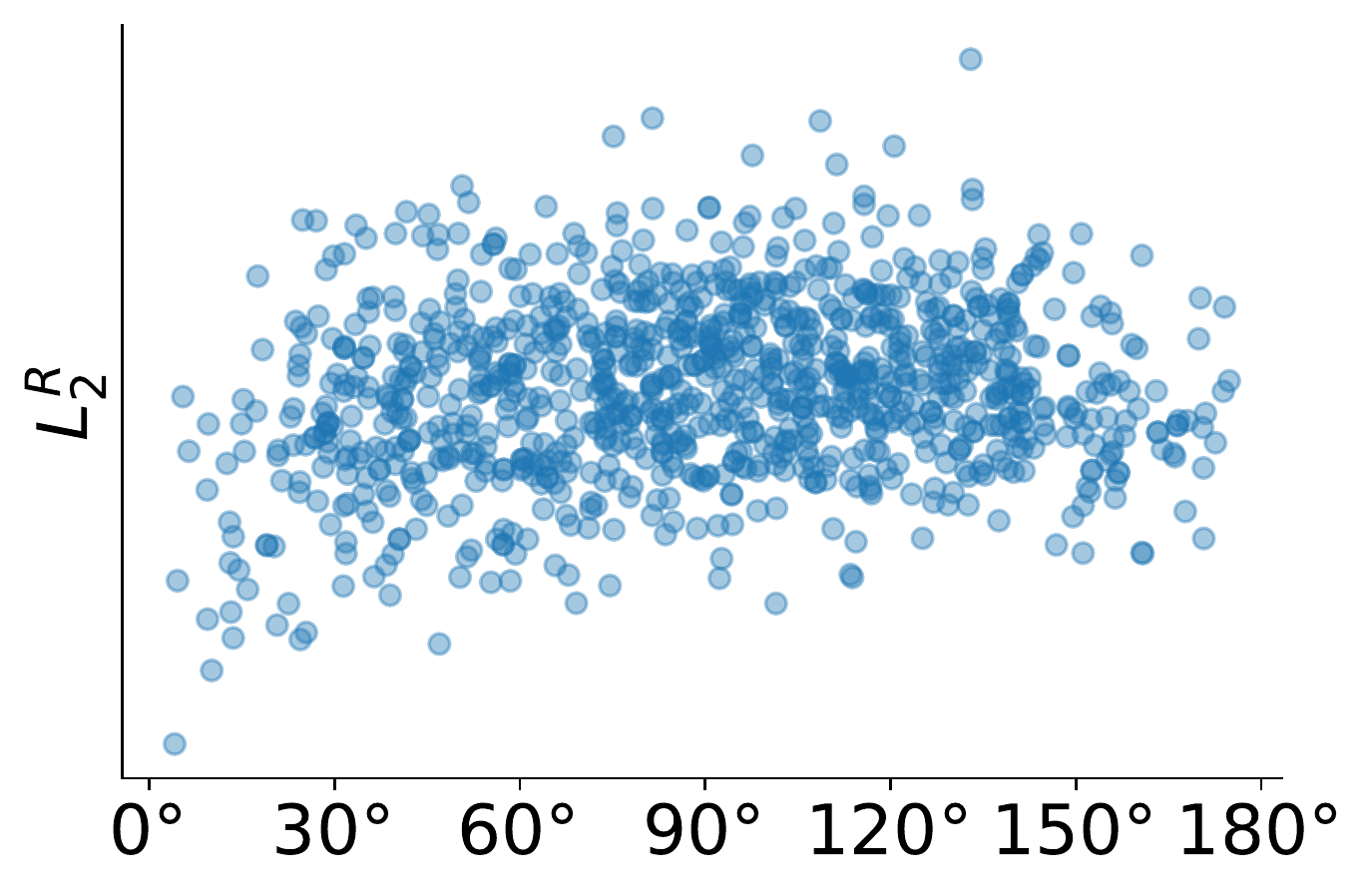}
    \includegraphics[width=0.4\textwidth]{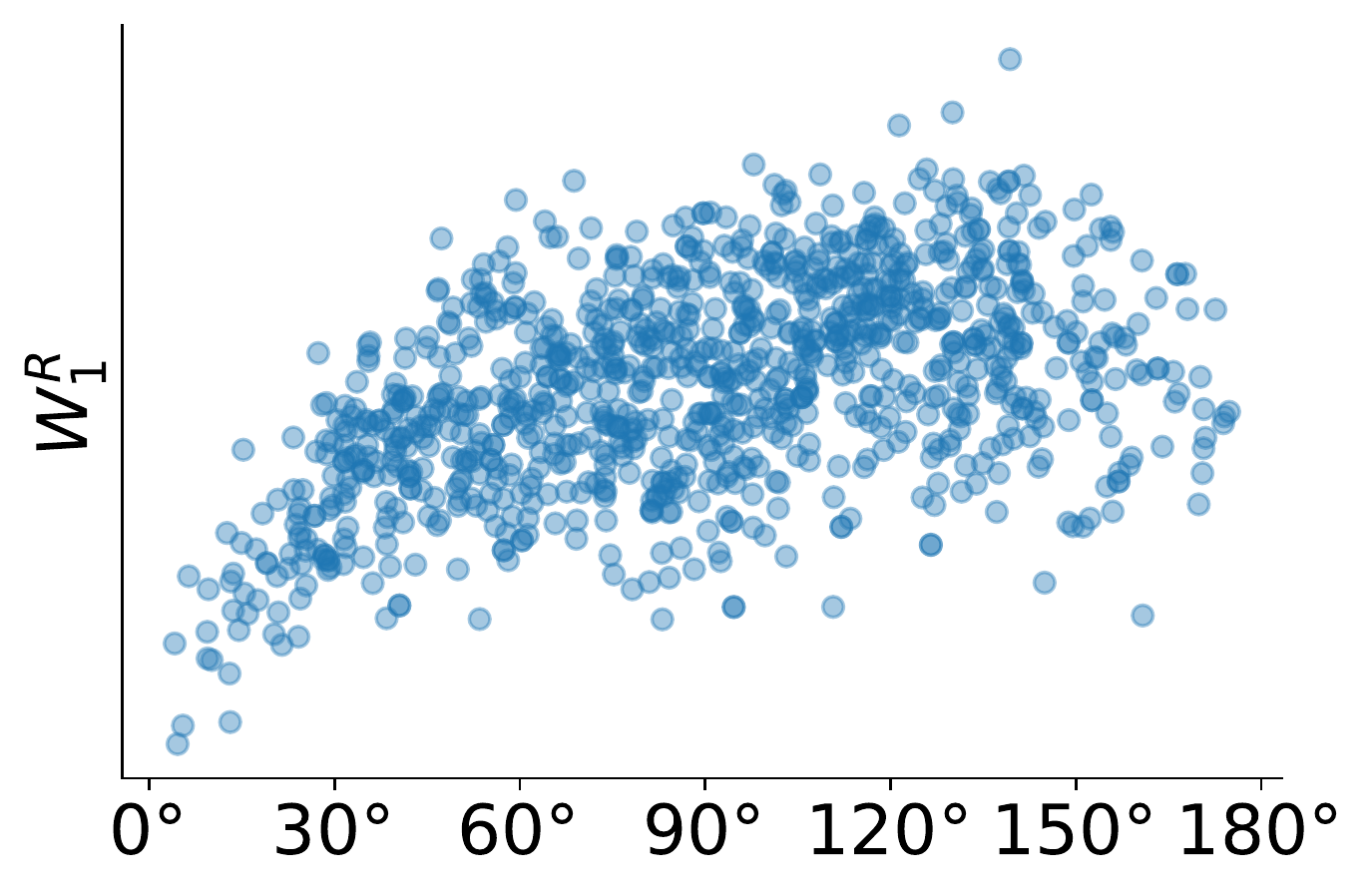}
    \caption{$W_1^R$ and $L_2^R$ distances versus the viewing angle difference between a fixed random clean projection and $1000$ random noisy projections (SNR=1/16)} 
    \label{fig:noise}
\end{figure}
\section{Theory} \label{sec:theory}

For a given particle, two tomographic projections from similar viewing angles typically produce images that are similar.
Hence, it is desirable that a metric for comparing tomographic projections will a assign small distance to projections that have a small difference in viewing angle.
We show that the rotationally-invariant Wasserstein metric satisfies this property.
\begin{proposition} \label{prop:w1UpperBound}
Let $\rho: \mathbb{R}^3 \rightarrow \mathbb{R}_{\geq 0}$ be a probability distribution supported on the 3D unit ball and
let $I_1$ and $I_2$ be its tomographic projections along the vectors $\u$ and $\v$ respectively.
Denote by $\measuredangle(\u,\v) \in [0,\pi]$ the angle between the vectors, then
\begin{align}
    W_p^R(I_1, I_2)^p \leq [2\sin(\measuredangle(\u,\v)/2)]^p \leq \measuredangle(\u, \v)^p
\end{align}
where $W_p^R$ is the rotationally-invariant Wasserstein metric defined in Eq. \eqref{def:rotinv_wasserstein}.
\end{proposition}

See Appendix \ref{sec:proofs} for the proof.
A similar upper-bound for the rotationally-invariant $L_2$ distance cannot hold for all densities $\rho$.
To see why, consider an off-center point mass. Any two projections at slightly different viewing angles will have a large $L_2^R$ distance no matter how small their angular difference is.
However, for densities with bounded gradients it is possible to produce upper bounds.
\begin{proposition} \label{prop:L2UpperBound}
Let $B = \sup_{\bf x}|\nabla\rho({\bf x})|$ be an upper bound on the absolute gradient of the density.
Under the same conditions of Proposition \ref{prop:w1UpperBound} we have
\begin{align}
    L_2^R(I_1, I_2) \leq 2\sqrt{\pi}B\measuredangle(\u, \v).
\end{align}
\label{eq:l2UpperBound}
\end{proposition}
The proof is in Appendix \ref{sec:proofs}.

This bound suggests that $L_2^R$ is a reasonable metric to use for very smooth signals. For non-smooth signals, or signals with very large $B$, this means that there is no guarantee that the $L_2^R$ distance will assign a small distance between projections with a small viewing angle. 

\section{Discussion}

From our numerical experiments, we  see that the rotationally-invariant Wasserstein-1 k-means clustering algorithm produces clusters that have more angular coherence and sharper cluster means than the rotationally-invariant $L_2$ clustering algorithm.
Thus, we believe it is a promising alternative to commonly used rotationally-invariant clustering algorithms based on the
$L_2$ distance.

Recently, there has been an explosion of interest in the analysis of molecular samples with continuous variability \citep{NakaneEtal2018,SorzanoEtal2019,MoscovichHaleviAndenSinger2020,LedermanAndenSinger2020,ZhongBeplerDavisBerger2020,DashtiEtal2020}.
In the presence of continuous variability, it is much more challenging to employ 2D class averaging and related techniques.
This is due to the fact that the clusters need to account not only for the entire range of viewing directions but also
for all the possible  3D conformations of the molecule of interest.
In future work we would like to test the performance of Wasserstein-based k-means clustering for datasets with continuous variability.
Wasserstein metrics seem like a natural fit since, by definition, motion of a part of the molecule incurs a distance no greater than the mass of the moving part multiplied by its displacement.

There are many other directions for future work, including incorporating Wasserstein metrics into EM-ML style algorithms \citep{Scheres2014},
Wasserstein barycenters for estimating the cluster centers \citep{CuturiDoucet2014},
rotationally-invariant denoising procedures \citep{ZhaoSinger2014},
incorporating translations, experiments on more realistic datasets with a contrast transfer function and analysis of real datasets.

\section{Broader impact}
Our work demonstrates that the use of Wasserstein metrics can improve the quality of clustering tomographic projections in the context of cryo-electron microscopy. This can reduce the number of images required to produce high quality cluster averages. Improved data efficiency in the cryo-EM pipeline will accelerate discovery in structural biology (with broad implications to medicine) and lower the barriers of entry to cryo-EM for resource-constrained labs around the world.
\appendix
\section{Appendix: proofs} \label{sec:proofs}

\subsection{Proof of Proposition \ref{prop:w1UpperBound}}
Consider a probability measure $\mu$ on some space $X$ and let $T:X \to Y$ be a measurable mapping.
it naturally induces a measure on its image known as the \emph{push-forward
measure} or \emph{image measure} $T \# \mu$, defined by
\begin{align}
    (T\#\mu)(S) = \mu(T^{-1}(S))
\end{align}

\begin{lemma} Given two probability distributions on the unit ball $B \subset \mathbb{R}^3$ $\mu, \eta$,
\begin{equation}
    W_1(P \# \mu, P \# \eta) \leq W_1(\mu, \eta)
\end{equation}
\label{lemma:contraction_emd}
where $P \# \mu$ is the tomographic projection in terms of a pushforward measure of $\mu$ defined on $B$ onto $\mathbb{D}
\subset \mathbb{R}^2$, and the Earthmover's distance for each is defined for the underlying probability space.\\
\end{lemma}
\begin{proof}
Let $\Gamma$ be any transportation measure on $B \times B$ such that $\Gamma(A, B) = \mu(A), \Gamma(B, A) = \nu(A)$. Let
$T$ map Lesbegue measurable subsets of $B \times B$ to measurable subsets of $\mathbb{D} \times \mathbb{D}$ such that each
element $((x_1, y_1, z_1),(x_2, y_2, z_2))$ in a measurable set of $B \times B$ is mapped to $((x_1, y_2), (x_2, y_2))$.
 Let $J = T \# \Gamma$ where $\#$ denotes the push-forward operator. $J(A,\mathbb{D}) = P \# \mu(A), J(\mathbb{D}, A) = P
\# \nu(A)$. By the change of variables formula for push-forward measures:
\begin{align}
    &\int_{\mathbb{D} \times \mathbb{D}}||(x_1, y_1) - (x_2, y_2)||dJ((x_1, y_1), (x_2, y_2))\\
    &=\int_{B \times B} ||(x_1, y_1) - (x_2, y_2)||d\Gamma((x_1,y_1,z_1),(x_2,y_2,z_2))\\
    &\leq \int_{B \times B} ||(x_1, y_1, z_1) - (x_2,y_2,z_2)||d\Gamma((x_1,y_1,z_1),(x_2,y_2,z_2)).
\end{align}

We  conclude that
\begin{align}
    W_1(\mu, \nu) &= \inf_\Gamma \int_{B \times B} ||u-v||d\Gamma(u,v) 
    \geq \inf_{T \# \Gamma} \int_{\mathbb{D} \times \mathbb{D}} ||u - v||d(T \# \Gamma)(u,v)
    \geq W_1(P \# \mu, P \# \nu),
\end{align}
where the last inequality is because any $J = T \# \Gamma$ is a valid transportation measure for $(P \# \mu, P \# \nu)$.
\end{proof}
\begin{lemma} Let $I_1, I_2$ be tomographic projections of $\rho$ at viewing angles $\u, \v$ respectively.
\begin{equation}
    W_1^R(I_1, I_2) \leq 2\sin \left( \measuredangle(\u, \v)/2 \right) \leq \measuredangle(\u, \v)
\end{equation}
\end{lemma}
\begin{proof} The second inequality is immediate, since for all $\theta \ge 0$ we have $\sin(\theta) \le \theta$. We denote the tomographic projection of $\rho$ at orientation $R \in \mathrm{SO}(3)$ as $\mathcal{T}_R\rho$. Without loss of generality, $I_1 = \mathcal{T}_I\rho$ where $I$ is the identity matrix and so $u = (0,0,-1)$ and $I_2 = \mathcal{T}_R\rho =\mathcal{T}_I(\rho \circ R^T)$.
decompose $R = R_1 \circ R_2$ to an in-plane rotation $R_1 \in \mathrm{SO}(2)$ and $R_2$  an out-of-plane rotation (has
axis in $\mathbb{R}^2$).

We observe the chain of inequalities:
\begin{align}
    W_1^R(I_1, I_2) &= W_1^R(I_1, \mathcal{T}_I((\rho \circ R_2^T) \circ R_1^T))\\
    &= W_1^R(I_1, \mathcal{T}_I(\rho \circ R_2^T))\\
    &\leq W_1 (I_1, \mathcal{T}_I(\rho \circ R_2^T))\\
    &=W_1(P \# \rho, P \# (\rho \circ R_2^T))\\
    &\leq W_1( \rho, (\rho \circ R_2^T))\\
    &\leq 2\sin \left( \measuredangle(\u,\v)/2 \right).
\end{align}

The first equality comes from the decomposition of $R$, the second comes from the invariance of $W_1^R$ to in-plane rotations, the third inequality comes from the definition of $W_1^R$, and the fourth equality rewrites $\mathcal{T}_I\rho = P \# \rho$. The fifth inequality comes from Lemma \ref{lemma:contraction_emd}. The final step of this chain comes from the Monge formulation of the Wasserstein metric \citep{PeyreCuturi2019}, from which it follows that for any two probability distributions on $B$, $P,Q$ we have
\begin{equation}
    W_1(P, Q) \leq \int_{B} \|x - F(x)\|dP(x).
\end{equation}
For $F: B \rightarrow B$,  $F \# P = Q$. 
It is immediate from this statement that using the map $F((x,y,z)) = R_2(x,y,z)$ we obtain
that \begin{equation}
    W_1(\rho, \rho \circ R_2^T) \leq \int_{B} ||(x,y,z) - R_2(x,y,z)||d\rho(x,y,z).
\end{equation}

Furthermore, for any vector $(x,y,z)$ with $\|(x,y,z)\| \le 1$, we have
\begin{equation}
    \|(x,y,z) - R_2(x,y,z)\| \leq 2\sin(\theta/2)
\end{equation}
where $\theta$ is the angle of rotation of $R_2$ in its axis-angle representation. This gives,
\begin{equation}
    W_1(\rho, \rho \circ R_2^T) \leq 2\sin \left(\measuredangle(\u,\v)/2\right).
\end{equation}
which completes our proof.
\end{proof}
\begin{corollary}
Let $I_1, I_2$ be tomographic projections of $\rho$ at viewing angles $\u, \v$ respectively.
\begin{equation}
    W_p^R(I_1, I_2) \leq \big[2\sin \left( \measuredangle(\u, \v)/2 \right)\big]^p \leq \measuredangle(\u, \v)^p
\end{equation}
\end{corollary}
Which comes from the fact that $W_p(I_1, I_2) \leq W_1(I_1, I_2)^p$ \citep{PeyreCuturi2019} from which we obtain $W_p^R(I_1, I_2) \leq W_1^R(I_1, I_2)^p$ and the subsequent chain of inequalities.
\subsection{Proof of Proposition \ref{prop:L2UpperBound}}
 Let $\rho$ be differentiable, and $|\nabla\rho| \leq B$. Then for two tomographic projections of $\rho$, $I_1, I_2$ at angles $\v,\u$ we have 
 \begin{equation}
     L_2^R(I_1, I_2) \leq 2\sqrt{\pi}B \measuredangle(\u, \v)
 \end{equation}

\begin{proof} This is a consequence of the mean value theorem. Without loss of generality, let us assume that $I_1 = \mathcal{T}_I(\rho)$ where $I$ is the identity matrix and so $v = (0,0,-1)$. 
Let $I_2= \mathcal{T}_{R} (\rho)$. We can decompose $R$ into an in-plane component $R_1 \in \mathrm{SO}(2)$
and an out-of-plane component $R_2$ such that $R = R_1 \circ R_2$ and $R_2$ has its axis of rotation in $\mathbb{R}^2$. Because $L_2^R(I_1, I_2)$
is invariant to rotations of $I_2$, $L_2^R(I_1, I_2) = L_2^R(I_1, \mathcal{T}_R(\rho
\circ R_1))$ which means that since $R^TR_1 = R_2^T$\\

\begin{align}
    L_2^R(I_1, I_2) &= L_2^R(I_1, \mathcal{T}_I(\rho \circ R_2^T))\\
    &\leq L_2(I_1, \mathcal{T}_I(\rho \circ R_2^T))\\
    &= \sqrt{\int_\mathbb{D}(\int_{-1}^1(\rho(x,y,z) - \rho \circ R_2^T(x,y,z))dz)^2}\\
    &\leq  \sqrt{\int_\mathbb{D} \left( \int_{-1}^1(\sup|\nabla \rho|) |\phi|dz \right)^2} \\
    &\leq 4\sqrt{\pi}B|\sin(\phi/2)| \leq 2\sqrt{\pi}B|\phi|.
\end{align}

The second to last line holds because by the mean value theorem
\begin{equation}
        ||\rho(x,y,z) - \rho \circ R_2^T(x,y,z)||\leq (\sup|\nabla \rho|)||(x,y,z) - R_2^T(x,y,z)||.
\end{equation}
We observe that the distance the point $(x,y,z)$ travels when acted on by $R_2^T$ is equal to the out-of-plane angle $\phi$
that $R_2^T$ rotates by multiplied by the distance from $(x,y,z)$ to the axis of $R_2^T$ which is $\leq 1$. Thus we can upper
bound $||(x,y,z) - R_2^T(x,y,z)||\leq |2 \sin(\phi/2)| \leq |\phi|$.
Since $|\phi| = \measuredangle(\u, \v)$ we achieve the desired upper bound.
\end{proof}

\addcontentsline{toc}{section}{References}
\bibliographystyle{plainnat-edited}
\small
\bibliography{wasserstein-k-means}

\begin{thebibliography}{40}
\providecommand{\natexlab}[1]{#1}
\providecommand{\url}[1]{\texttt{#1}}
\expandafter\ifx\csname urlstyle\endcsname\relax
  \providecommand{\doi}[1]{doi: #1}\else
  \providecommand{\doi}{doi: \begingroup \urlstyle{rm}\Url}\fi

\bibitem[Berman(2000)]{BermanEtal2000}
Helen~M. Berman.
\newblock {The Protein Data Bank}.
\newblock \emph{Nucleic Acids Research}, 28\penalty0 (1):\penalty0 235--242,
  2000.
\newblock \doi{10.1093/nar/28.1.235}.

\bibitem[Chen and Grigorieff(2007)]{ChenGrigorief2007}
James~Z. Chen and Nikolaus Grigorieff.
\newblock {SIGNATURE: A single-particle selection system for molecular electron
  microscopy}.
\newblock \emph{Journal of Structural Biology}, 157\penalty0 (1):\penalty0
  168--173, 2007.
\newblock \doi{10.1016/j.jsb.2006.06.001}.

\bibitem[Cheng(2018)]{Cheng2018}
Yifan Cheng.
\newblock {Single-particle cryo-EM—How did it get here and where will it go}.
\newblock \emph{Science}, 361\penalty0 (6405):\penalty0 876--880, 2018.
\newblock \doi{10.1126/science.aat4346}.

\bibitem[Cressey and Callaway(2017)]{CryoEMNobel2017}
Daniel Cressey and Ewen Callaway.
\newblock {Cryo-electron microscopy wins chemistry Nobel}.
\newblock \emph{Nature}, 550\penalty0 (7675):\penalty0 167--167, 2017.
\newblock \doi{10.1038/nature.2017.22738}.

\bibitem[Cuturi and Doucet(2014)]{CuturiDoucet2014}
Marco Cuturi and Arnaud Doucet.
\newblock {Fast computation of Wasserstein barycenters}.
\newblock \emph{International Conference on Machine Learning (ICML)},
  32\penalty0 (2):\penalty0 685--693, 2014.

\bibitem[Dashti et~al.(2020)Dashti, Mashayekhi, Shekhar, {Ben Hail}, Salah,
  Schwander, des Georges, Singharoy, Frank, and Ourmazd]{DashtiEtal2020}
Ali Dashti et~al.
\newblock {Retrieving functional pathways of biomolecules from single-particle
  snapshots}.
\newblock \emph{Nature Communications}, 11\penalty0 (1):\penalty0 4734, 2020.
\newblock \doi{10.1038/s41467-020-18403-x}.

\bibitem[de~la Rosa-Trev{\'{i}}n et~al.(2013)de~la Rosa-Trev{\'{i}}n,
  Ot{\'{o}}n, Marabini, Zald{\'{i}}var, Vargas, Carazo, and Sorzano]{Xmipp2013}
J.M. de~la Rosa-Trev{\'{i}}n et~al.
\newblock {Xmipp 3.0: An improved software suite for image processing in
  electron microscopy}.
\newblock \emph{Journal of Structural Biology}, 184\penalty0 (2):\penalty0
  321--328, 2013.
\newblock \doi{10.1016/j.jsb.2013.09.015}.

\bibitem[Flamary and Courty(2017)]{FlameryCourty2017}
R{\'{e}}mi Flamary and Nicolas Courty.
\newblock {{POT} Python Optimal Transport library}, 2017.
\newblock \url{http://pythonot.github.io}.

\bibitem[Frank(2006)]{Frank2006}
Joachim Frank.
\newblock \emph{{Three-Dimensional Electron Microscopy of Macromolecular
  Assemblies}}.
\newblock Oxford University Press, 2006.
\newblock ISBN 9780195182187.
\newblock \doi{10.1093/acprof:oso/9780195182187.001.0001}.

\bibitem[Frank and Wagenknecht(1983)]{FrankWagenkrecht1983}
Joachim Frank and Terence Wagenknecht.
\newblock {Automatic selection of molecular images from electron micrographs}.
\newblock \emph{Ultramicroscopy}, 12\penalty0 (3):\penalty0 169--175, 1983.
\newblock \doi{10.1016/0304-3991(83)90256-5}.

\bibitem[Grant et~al.(2018)Grant, Rohou, and
  Grigorieff]{GrantRohouGrigorieff2018}
Timothy Grant, Alexis Rohou and Nikolaus Grigorieff.
\newblock {cisTEM, user-friendly software for single-particle image
  processing}.
\newblock \emph{eLife}, 7\penalty0 (3):\penalty0 377--388, 2018.
\newblock \doi{10.7554/eLife.35383}.

\bibitem[Greenberg and Shkolnisky(2017)]{GreenbergShkolnisky2017}
Ido Greenberg and Yoel Shkolnisky.
\newblock {Common lines modeling for reference free Ab-initio reconstruction in
  cryo-EM}.
\newblock \emph{Journal of Structural Biology}, 200\penalty0 (2):\penalty0
  106--117, 2017.
\newblock \doi{10.1016/j.jsb.2017.09.007}.

\bibitem[Heimowitz et~al.(2018)Heimowitz, And{\'{e}}n, and
  Singer]{HeimowitzAndenSinger2018}
Ayelet Heimowitz, Joakim And{\'{e}}n and Amit Singer.
\newblock {APPLE picker: Automatic particle picking, a low-effort cryo-EM
  framework}.
\newblock \emph{Journal of Structural Biology}, 204\penalty0 (2):\penalty0
  215--227, 2018.
\newblock \doi{10.1016/j.jsb.2018.08.012}.

\bibitem[Lederman et~al.(2020)Lederman, And{\'{e}}n, and
  Singer]{LedermanAndenSinger2020}
Roy~R. Lederman, Joakim And{\'{e}}n and Amit Singer.
\newblock {Hyper-molecules: on the representation and recovery of dynamical
  structures for applications in flexible macro-molecules in cryo-EM}.
\newblock \emph{Inverse Problems}, 36\penalty0 (4):\penalty0 044005, 2020.
\newblock \doi{10.1088/1361-6420/ab5ede}.

\bibitem[Lee et~al.(2019)Lee, Gommers, Waselewski, Wohlfahrt, and
  O'Leary]{LeeEtal2019}
G.~Lee, R.~Gommers, F.~Waselewski, K.~Wohlfahrt and A.~O'Leary.
\newblock {PyWavelets: a Python package for wavelet analysis}.
\newblock \emph{Journal of Open Source Software}, 4\penalty0 (36):\penalty0
  1237, 2019.
\newblock \doi{10.21105/joss.01237}.

\bibitem[Lyumkis(2019)]{Lyumkis2019}
Dmitry Lyumkis.
\newblock {Challenges and opportunities in cryo-EM single-particle analysis}.
\newblock \emph{Journal of Biological Chemistry}, 294\penalty0 (13):\penalty0
  5181--5197, 2019.
\newblock \doi{10.1074/jbc.REV118.005602}.

\bibitem[Lyumkis et~al.(2013)Lyumkis, Brilot, Theobald, and
  Grigorieff]{LyumkisEtal2013}
Dmitry Lyumkis, Axel~F. Brilot, Douglas~L. Theobald and Nikolaus Grigorieff.
\newblock {Likelihood-based classification of cryo-EM images using FREALIGN}.
\newblock \emph{Journal of Structural Biology}, 183\penalty0 (3):\penalty0
  377--388, 2013.
\newblock \doi{10.1016/j.jsb.2013.07.005}.

\bibitem[Moscovich et~al.(2020)Moscovich, Halevi, And{\'{e}}n, and
  Singer]{MoscovichHaleviAndenSinger2020}
Amit Moscovich, Amit Halevi, Joakim And{\'{e}}n and Amit Singer.
\newblock {Cryo-EM reconstruction of continuous heterogeneity by Laplacian
  spectral volumes}.
\newblock \emph{Inverse Problems}, 36\penalty0 (2):\penalty0 024003, 2020.
\newblock \doi{10.1088/1361-6420/ab4f55}.

\bibitem[Nakane et~al.(2018)Nakane, Kimanius, Lindahl, and
  Scheres]{NakaneEtal2018}
Takanori Nakane, Dari Kimanius, Erik Lindahl and Sjors~HW Scheres.
\newblock {Characterisation of molecular motions in cryo-EM single-particle
  data by multi-body refinement in RELION}.
\newblock \emph{eLife}, 7:\penalty0 1--18, 2018.
\newblock \doi{10.7554/eLife.36861}.

\bibitem[Penczek et~al.(1992)Penczek, Radermacher, and
  Frank]{PenczekRadermacherFrank1992}
Pawel Penczek, Michael Radermacher and Joachim Frank.
\newblock {Three-dimensional reconstruction of single particles embedded in
  ice}.
\newblock \emph{Ultramicroscopy}, 40\penalty0 (1):\penalty0 33--53, 1992.
\newblock \doi{10.1016/0304-3991(92)90233-A}.

\bibitem[Penczek et~al.(1996)Penczek, Zhu, and Frank]{PenczekZhuFrank1996}
Pawel~A. Penczek, Jun Zhu and Joachim Frank.
\newblock {A common-lines based method for determining orientations for N>3
  particle projections simultaneously}.
\newblock \emph{Ultramicroscopy}, 63\penalty0 (3-4):\penalty0 205--218, 1996.
\newblock \doi{10.1016/0304-3991(96)00037-X}.

\bibitem[Peyr{\'{e}} and Cuturi(2019)]{PeyreCuturi2019}
Gabriel Peyr{\'{e}} and Marco Cuturi.
\newblock {Computational Optimal Transport: With Applications to Data Science}.
\newblock \emph{Foundations and Trends{\textregistered} in Machine Learning},
  11\penalty0 (5-6):\penalty0 355--607, 2019.
\newblock \doi{10.1561/2200000073}.

\bibitem[Punjani et~al.(2017)Punjani, Rubinstein, Fleet, and
  Brubaker]{PunjaniEtal2017}
Ali Punjani, John~L. Rubinstein, David~J. Fleet and Marcus~A. Brubaker.
\newblock {cryoSPARC: algorithms for rapid unsupervised cryo-EM structure
  determination}.
\newblock \emph{Nature Methods}, 14\penalty0 (3):\penalty0 290--296, 2017.
\newblock \doi{10.1038/nmeth.4169}.

\bibitem[Rohou and Grigorieff(2015)]{RohouGrigorieff2015}
Alexis Rohou and Nikolaus Grigorieff.
\newblock {CTFFIND4: Fast and accurate defocus estimation from electron
  micrographs}.
\newblock \emph{Journal of Structural Biology}, 192\penalty0 (2):\penalty0
  216--221, 2015.
\newblock \doi{10.1016/j.jsb.2015.08.008}.

\bibitem[Schatz and {Van Heel}(1990)]{SchatzVanheel1990}
Michael Schatz and Marin {Van Heel}.
\newblock {Invariant classification of molecular views in electron
  micrographs}.
\newblock \emph{Ultramicroscopy}, 32\penalty0 (3):\penalty0 255--264, 1990.
\newblock \doi{10.1016/0304-3991(90)90003-5}.

\bibitem[Scheres(2012)]{Scheres2012b}
Sjors~H.W. Scheres.
\newblock {RELION: Implementation of a Bayesian approach to cryo-EM structure
  determination}.
\newblock \emph{Journal of Structural Biology}, 180\penalty0 (3):\penalty0
  519--530, 2012.
\newblock \doi{10.1016/j.jsb.2012.09.006}.

\bibitem[Scheres(2015)]{Scheres2014}
Sjors~H.W. Scheres.
\newblock {Semi-automated selection of cryo-EM particles in RELION-1.3}.
\newblock \emph{Journal of Structural Biology}, 189\penalty0 (2):\penalty0
  114--122, 2015.
\newblock \doi{10.1016/j.jsb.2014.11.010}.

\bibitem[Scheres et~al.(2005)Scheres, Valle, and
  Carazo]{ScheresValleCarazo2005}
Sjors~H.W. Scheres, Mikel Valle and J.-M. Carazo.
\newblock {Fast maximum-likelihood refinement of electron microscopy images}.
\newblock \emph{Bioinformatics}, 21\penalty0 (Suppl 2):\penalty0 ii243--ii244,
  2005.
\newblock \doi{10.1093/bioinformatics/bti1140}.

\bibitem[Shirdhonkar and Jacobs(2008)]{ShirdhonkarJacobs2008}
Sameer Shirdhonkar and David~W. Jacobs.
\newblock {Approximate earth mover's distance in linear time}.
\newblock In \emph{Conference on Computer Vision and Pattern Recognition
  (CVPR)}, pages 1--8. IEEE, 2008.
\newblock \doi{10.1109/CVPR.2008.4587662}.

\bibitem[Singer and Sigworth(2020)]{SingerSigworth2020}
Amit Singer and Fred~J. Sigworth.
\newblock {Computational Methods for Single-Particle Electron Cryomicroscopy}.
\newblock \emph{Annual Review of Biomedical Data Science}, 3\penalty0
  (1):\penalty0 163--190, 2020.
\newblock \doi{10.1146/annurev-biodatasci-021020-093826}.

\bibitem[Sorzano et~al.(2019)Sorzano, Jim{\'{e}}nez, Mota, Vilas, Maluenda,
  Mart{\'{i}}nez, Ram{\'{i}}rez-Aportela, Majtner, Segura,
  S{\'{a}}nchez-Garc{\'{i}}a, Rancel, del Ca{\~{n}}o, Conesa, Melero, Jonic,
  Vargas, Cazals, Freyberg, Krieger, Bahar, Marabini, and
  Carazo]{SorzanoEtal2019}
C.~O.~S. Sorzano et~al.
\newblock {Survey of the analysis of continuous conformational variability of
  biological macromolecules by electron microscopy}.
\newblock \emph{Acta Crystallographica Section F Structural Biology
  Communications}, 75\penalty0 (1):\penalty0 19--32, 2019.
\newblock \doi{10.1107/S2053230X18015108}.

\bibitem[Sorzano et~al.(2010)Sorzano, Bilbao-Castro, Shkolnisky, Alcorlo,
  Melero, Caffarena-Fern{\'{a}}ndez, Li, Xu, Marabini, and
  Carazo]{SorzanoEtal2010}
C.O.S. Sorzano et~al.
\newblock {A clustering approach to multireference alignment of single-particle
  projections in electron microscopy}.
\newblock \emph{Journal of Structural Biology}, 171\penalty0 (2):\penalty0
  197--206, 2010.
\newblock \doi{10.1016/j.jsb.2010.03.011}.

\bibitem[Tang et~al.(2007)Tang, Peng, Baldwin, Mann, Jiang, Rees, and
  Ludtke]{TangEtal2007}
Guang Tang et~al.
\newblock {EMAN2: An extensible image processing suite for electron
  microscopy}.
\newblock \emph{Journal of Structural Biology}, 157\penalty0 (1):\penalty0
  38--46, 2007.
\newblock \doi{10.1016/j.jsb.2006.05.009}.

\bibitem[Villani(2003)]{Villani2003}
C{\'{e}}dric Villani.
\newblock \emph{{Topics in Optimal Transportation}}, volume~58 of
  \emph{Graduate Studies in Mathematics}.
\newblock American Mathematical Society, Providence, Rhode Island, 2003.
\newblock ISBN 9780821833124.
\newblock \doi{10.1090/gsm/058}.

\bibitem[Vinothkumar and Henderson(2016)]{VinothkumarHenderson2016}
Kutti~R. Vinothkumar and Richard Henderson.
\newblock {Single particle electron cryomicroscopy: trends, issues and future
  perspective}.
\newblock \emph{Quarterly Reviews of Biophysics}, 49:\penalty0 1--25, 2016.
\newblock \doi{10.1017/S0033583516000068}.

\bibitem[Wong et~al.(2014)Wong, Bai, Brown, Fernandez, Hanssen, Condron, Tan,
  Baum, and Scheres]{WongEtal2014}
Wilson Wong et~al.
\newblock {Cryo-EM structure of the Plasmodium falciparum 80S ribosome bound to
  the anti-protozoan drug emetine}.
\newblock \emph{eLife}, 3\penalty0 (3):\penalty0 1--20, 2014.
\newblock \doi{10.7554/eLife.03080}.

\bibitem[Zelesko et~al.(2020)Zelesko, Moscovich, Kileel, and
  Singer]{ZeleskoMoscovichKileelSinger2020}
Nathan Zelesko, Amit Moscovich, Joe Kileel and Amit Singer.
\newblock {Earthmover-Based Manifold Learning for Analyzing Molecular
  Conformation Spaces}.
\newblock In \emph{International Symposium on Biomedical Imaging (ISBI)}, pages
  1715--1719. IEEE, 2020.
\newblock \doi{10.1109/ISBI45749.2020.9098723}.

\bibitem[Zhao and Singer(2014)]{ZhaoSinger2014}
Zhizhen Zhao and Amit Singer.
\newblock {Rotationally invariant image representation for viewing direction
  classification in cryo-EM}.
\newblock \emph{Journal of Structural Biology}, 186\penalty0 (1):\penalty0
  153--166, 2014.
\newblock \doi{10.1016/j.jsb.2014.03.003}.

\bibitem[Zhong et~al.(2020)Zhong, Bepler, Davis, and
  Berger]{ZhongBeplerDavisBerger2020}
Ellen~D Zhong, Tristan Bepler, Joseph~H Davis and Bonnie Berger.
\newblock {Reconstructing continuous distributions of 3D protein structure from
  cryo-{EM} images}.
\newblock In \emph{International Conference on Learning Representations
  (ICLR)}, pages 1--20, 2020.

\bibitem[Zhou et~al.(2020)Zhou, Moscovich, Bendory, and
  Bartesaghi]{ZhouMoscovichBendoryBartesaghi2020}
Ye~Zhou, Amit Moscovich, Tamir Bendory and Alberto Bartesaghi.
\newblock {Unsupervised particle sorting for high-resolution single-particle
  cryo-EM}.
\newblock \emph{Inverse Problems}, 36\penalty0 (4), 2020.
\newblock \doi{10.1088/1361-6420/ab5ec8}.

\end{thebibliography}

\end{document}